\documentclass[twoside]{article}
\usepackage{aistats2020}
\usepackage{graphicx}
\usepackage{amsfonts,amsthm,amsmath,amssymb}       % blackboard math symbols
\newcommand {\OMIT}[1]{}               % to ignore things

\newtheorem{prop}{Proposition}
\usepackage{algorithm}
\usepackage[noend]{algpseudocode}
 \usepackage{enumitem}
 \usepackage{xcolor}

\renewcommand{\epsilon}{\varepsilon}

\newcommand {\br}[1]{\left(#1\right)}

\newcommand {\cbr}[1]{\left\{#1 \right\}}

\newcommand{\RR}{\mathbb{R}}    %reals
\newcommand{\NN}{\mathbb{N}}    %integer
\newcommand{\Dcal}{\mathcal{D}} %sigma-algebra
\newcommand{\Lcal}{\mathcal{L}}

\newcommand{\Pcal}{\mathcal{P}}

\newcommand{\Xcal}{\mathcal{X}} %set of possible inputs x
\newcommand{\argmin}{\mathop{\mathrm{argmin}\,}}

\newcommand{\norm}[1]{\left\|{#1}\right\|}

\renewcommand\phi\varphi

\newcommand\One{\mathbf{1}}

\newcommand{\eqdef}{\ensuremath{\stackrel{\mbox{\upshape\tiny def.}}{=}}}

\usepackage[round]{natbib}

\begin{document}

\twocolumn[

\aistatstitle{Differentiable Deep Clustering with Cluster Size Constraints}
%\aistatstitle{Optimal Transport for Clustering High Dimensional Data with Cluster Constraints}

%\aistatsauthor{Anonymous}
\aistatsauthor{ Aude Genevay \And Gabriel Dulac-Arnold \And  Jean-Philippe Vert}

%\aistatsaddress{Anonymous} ]
\aistatsaddress{ CSAIL, MIT \And  Google Brain \And Google Brain } ]

\begin{abstract}
Clustering is a fundamental unsupervised learning approach.  Many clustering algorithms -- such as $k$-means -- rely on the euclidean distance as a similarity measure, which is often not the most relevant metric for high dimensional data such as images. Learning a lower-dimensional embedding that can better reflect the geometry of the dataset is therefore instrumental for performance. We propose a new approach for this task where the embedding is performed by a differentiable model such as a deep neural network. By rewriting the $k$-means clustering algorithm as an optimal transport task, and adding an entropic regularization, we derive a fully differentiable loss function that can be minimized with respect to both the embedding parameters and the cluster parameters via stochastic gradient descent. We show that this new formulation generalizes a recently proposed state-of-the-art method based on soft-$k$-means by adding constraints on the cluster sizes. Empirical evaluations on image classification benchmarks suggest that compared to state-of-the-art methods, our optimal transport-based approach provide better unsupervised accuracy and does not require a pre-training phase.
\end{abstract}

\section{Introduction}
Clustering is a fundamental unsupervised learning task, where, given a training set of data $(x_1, \ldots, x_n) \subset \Xcal$ and a number of classes $K$, we aim to partition the training data into $K$ non-overlapping clusters corresponding to different classes of points. We consider an extension of this problem where we additionally want to classify out-of-sample data not in the training set, i.e., we want to infer a function $c:\Xcal\rightarrow [1,K]$ that maps a given point in the data space to a class. Many clustering techniques exist, such as agglomerative clustering, when $\Xcal$ is endowed with a metric, or $k$-means, when $\Xcal=\RR^d$. However, these methods often fail when applied to complex higher-dimentional data, as the usual metric on the dataspace (e.g. euclidean in $\RR^d$) is not meaningful.

For complex data such as images or strings, recent years have witnessed significant progress in learning representations $f_\theta : \Xcal \rightarrow \RR^p$, where $f_\theta$ is typically a deep neural network (DNN) with parameters $\theta$, and  $\RR^p$ is a low dimensional space that intends to capture the underlying structure of the data \citep{Bengio2013Representation}. The representation $f_\theta$ is usually optimized to either solve a supervised task such as image classification using a training set of labeled images, or, in the absence of labels, can be used to summarize the data, e.g., by using (variational) auto-encoders or GANs \citep{Kingma2014,gan}.

Any such learned representation $f_\theta$ can be used in conjunction with any clustering algorithm to cluster the training set mapped to the representation space $ (f_\theta(x_1), \ldots, f_\theta(x_n)) \subset \RR^p$. However, there is no guarantee that the representation $f_\theta$ is ``good'' for this clustering task if it has been optimized for another task. In this work, we propose a new approach to learn a representation $f_\theta$ well adapted to solve the clustering task, in the absence of labels. This setting has been considered by several authors recently, often under the name of ``deep clustering'', and before reviewing existing approaches let us fix some notations.

\paragraph{Setting and notations.}
 We consider a dataset of $n$ points $\Dcal_n = (x_1, \dots, x_n) \in \Xcal ^n$ where $\Xcal \in \RR^d$ is the space of input data (e.g., $\Xcal=\RR^{32\times 32 \times 3}$ for 3-channel $32\times 32$ images). Our goal is to cluster the data into $K$ clusters, which might correspond to the number of classes in a supervised setting. For a dataset $\Dcal_n$ we denote by $f_\theta(\Dcal_n)$ the embedded dataset  $(f_\theta(x_1), \dots, f_\theta(x_n))$, where $f_\theta: \Xcal \rightarrow \RR^p$ is a deep neural network with parameters $\theta$ and $p<<d$.
 
 Additionally,  we denote by $\Delta_K =\cbr{z \in \RR_+^K\,:\, \sum_{k=1}^K z_k = 1}$ the probability simplex. For any integer $n\in\NN$, let $\One_n\in\RR^n$ be the $n$-dimensional vector of ones. Given two vectors $a,b\in\RR^n$, where $b_i\neq 0$ for $i\in[1,n]$, we denote by $a\oslash b \in \RR^n$ the vector with entries $(a\oslash b)_i = a_i / b_i$. For any vector or matrix $M$, we denote by  $\exp(M)$  the matrix obtained by applying the operation entrywise, e.g., $[\exp(M)]_{ij} = \exp(M_{ij})$, and by $M^\top$ the transpose of $M$.

\paragraph{A discriminative approach to clustering.}
The clustering task can be recast as a classification problem, without relying on a representation of the data, focusing directly on the clustering function $c_\theta:\Xcal\rightarrow [1,K]$. This only makes sense in cases where classes are well separated. In that case, the labels $Y=(y_1,\ldots,y_n)\in[1,K]^n$ of the training set are known, then one could train a supervised model by minimizing over $\theta$ an empirical risk of the form
$$
R(\theta,Y) = \frac{1}{n}\sum_{i=1}^n \ell\br{c_{\theta}(x_i),y_i}\,.
$$
Since we are considering an unsupervised setting, $Y$ is not available. A solution is thus to jointly optimize the above criterion over $\theta$ and $Y$ to learn both a class assignment and a classifier:
\begin{equation}\label{eq:minmin}
(\hat{\theta},\hat{Y}) \in \argmin_{\theta \in \Theta,Y \in [1,K]^n} R(\theta,Y).
\end{equation}
An obvious computational difficulty is that this problem involves the discrete variable $Y$. Besides, some kind of regularization is required in this double optimization task to prevent trivial solutions; adding constraints on $Y$ is crucial to prevent empty or overpopulated clusters. \citet{Joulin2010Discriminative} propose a convex relaxation of (\ref{eq:minmin}) in the case of linear regression with the squared loss $\ell(u,v)=(u-v)^2$ for binary problems ($K=2$). In that case, the objective function is quadratic in $Y$ and they use the standard semidefinite programming (SDP) relaxation for the matrix $YY^\top$ to approximate a minimum.

A different approach is used by \citet{chang2017deep} who recast the clustering problem as a binary classification problem: given two data points, do they belong to the same cluster? The resulting algorithm, Deep Adaptive Clustering (DAC) can be summarized as follows: each data point is mapped to a vector in the unit ball of $\RR^K$ thanks to $f_\theta$, which represents its probability to belong to each class. These probabilities are then compared with the cosine distance, which defines the similarity of the two points. The points are assumed to belong to the same class if the similarity is above a certain threshold. The parameters $\theta$ are then updated in order to increase similarity between similar points.

\paragraph{Learning a 'clustering-friendly' representation.}
However, representations are useful beyond the clustering task, e.g., to extract features or reduce the dimension of a dataset, which is why many methods in the literature rather learn a representation $f_\theta:\Xcal\rightarrow \RR^p$ such that $f_\theta(\Dcal_n)$ becomes easy to cluster. In most cases, the problem consists in minimizing a combined loss made of two terms :  (i) a ``representation loss'' $\ell_{r}(\theta)$ to ensure that the representation space is not degenerate
(ii) a ``clustering loss'' $\ell_c(\theta)$ to enforce that the learned representation $f_\theta(\Dcal)$ is relevant for the clustering task.
This yields the following optimization problem:
\begin{equation}\label{eq:combloss}\tag{$\Pcal$}
\min_\theta L(\theta) \eqdef \ell_{r}(\theta) + \lambda \ell_c(\theta)\,.
\end{equation}
While the choice of the reconstruction loss of an auto-encoder (with encoder $f_\theta$ and decoder $g_\phi$)
\begin{equation}\label{eq:lossAE}
\ell_{r}^{ae}(\theta) = \min_\phi \sum_{i=1}^n \|x_i - g_\phi\br{f_\theta(x_i)}\|^2\,,
\end{equation}
is the standard for the representation loss, these methods vary mostly in their choices of the clustering losses, the auto-encoder model, and the optimization strategies (in particular to prevent trivial solutions).

\citet{Song2013Auto} are one of the earliest to learn a representation for clustering by tweaking the objective function of a standard auto-encoder. They formulate the problem as minimizing the combined loss \eqref{eq:combloss} with the objective of $k$-means as the clustering loss:
\begin{equation}\label{eq:losskm}
\ell_c^{km}(\theta) = \min_{\mu_1,\ldots,\mu_k \in \RR^p} \cbr{ \sum_{i=1}^n \min_{j\in[1,K]}\| f_\theta(x_i) - \mu_j \|^2 }\,.
\end{equation} To optimize the objective function over the encoder parameters $\theta$, the decoder parameters $\phi$ and the cluster centers $\mu_1,\ldots,\mu_k$, they alternate one epoch of stochastic gradient descent over $(\theta,\phi)$, with one update of the cluster centers and assignments.

While most state-of-the-art methods rely on clustering objectives that are strongly linked to $k$-means, Joint Unsupervised LEarning (JULE) \citep{Yang_2016} uses a clustering loss based on ``agglomerative clustering''. Starting from clusters consisting of datapoints, the training alternates between a few steps of agglomerative clustering, i.e., merging similar clusters, and a backward pass during which the network parameters are updated to minimize the clustering loss. Although this method has a more flexible geometry, it requires building an affinity graph of the dataset after each update and is thus computationally heavy.

\citet{Xie2016Unsupervised} propose Deep Embedded Clustering (DEC)~ which starts with a pre-training phase using only the reconstruction loss $\ell_{r}(\theta,\phi)$ and then improves the clustering ability of the representation by optimizing $f_\theta$ in a self-supervised manner. Their clustering loss is the Kullback-Leibler divergence between the soft-assignments $q_{ik}$ of each point $i$ to each cluster $k$ and the square of the soft-assignments, which should push the embedding to favor harder assignments. There are several variants of DEC using more sophisticated auto-encoders and training techniques such as \citet{Guo2017Improved}. The DEPICT algorithm \citep{Dizaji2017Deep} similarly minimize the KL divergence to sharpen their assignments but also introduce a classifier $h_\beta$, that outputs $h_\beta(z)$ a probability distribution over the $k$ classes (typically, a neural network with softmax activation at the last layer). Thus the clustering loss corresponds to the possibility to discriminate the data in $k$ different classes.

The clustering loss in Deep Clustering Network (DCN) \citep{Yang2017Towards} is the objective of $k$-means in the representation space. However, minimizing the total loss $L$ over $\theta,\phi,\mu$ (cluster centers) and $\pi$ (cluster assignments) jointly is challenging. Thus, \citet{Yang2017Towards} alternate optimization in $(\theta,\phi)$ for fixed $(\mu,\pi)$, which becomes a variant of AE training, and in $(\mu,\pi)$ for fixed $(\theta,\phi)$. The Deep $k$-means (DKM) \citep{fard2018deep} algorithm uses the same loss as DCN but relax the assignment problem by replacing the cluster assignments with soft-assignments in the $k$-means objective. 
This results in a  clustering loss can be jointly minimized over $\theta$ and $\mu$, using stochastic gradient descent (SGD), and leads to state-of-the-art performance in deep clustering \citep{fard2018deep}. The latter is the approach which is closest to ours, as we also propose a fully differentiable objective based on $k$-means.

\paragraph{Clustering and optimal transport} There is a link between $k$-means clustering and optimal transport, which was first noticed in \citet{pollard1982quantization} and studied in more details in \citet{canas2012learning}. Roughly speaking, optimal transport is equivalent to a constrained formulation of $k$-means in which the cluster sizes are prescribed. This framework makes sense in a setting where the proportion of each class in a dataset is known, but no information is available at the individual level. \citet{cuturi2014fast} introduced an entropic regularization of that problem which allows for an efficient solver.

\paragraph{Contributions} Following \citet{cuturi2014fast} we exploit the connection between optimal transport and $k$-means, and rely on entropic regularization to derive a fully-differentiable clustering loss that can be used in \eqref{eq:combloss} and directly optimized with SGD. We give an insight on the effect of regularization in the cluster assignment problem, and show that the soft $k$-means loss introduced by \citet{fard2018deep} can be interpreted as an optimal transport loss with only one marginal constraint. The constraints on cluster sizes that naturally occur with optimal transport allow to enforce a prior on cluster sizes without relying to additional terms in the optimization problem. This leads to better clustering performance on benchmark datasets. 

\section{Clustering with Optimal Transport}

\paragraph{Cluster assignment as an  optimal transport problem}
Consider $n$ sample points $\cbr{x_1, \ldots, x_n} \subset \RR^d$ embedded in the representation space via $f_\theta : \RR^d \rightarrow \RR^p $, and $K$ clusters in that representation space with centers $\cbr{\mu_1, \ldots, \mu_K} \subset \RR^p$. We want to assign samples to clusters so that:
\begin{enumerate}[label=(\roman*)]
\item each sample is assigned to exactly one cluster,
\item each cluster $k=1,\ldots,K$ contains exactly $n_k$ points,
\item the total distance (in the representation space) between cluster centers and their assigned samples is minimal.
\end{enumerate}
The mathematical formulation of the above problem reads as follows:
\begin{subequations}
\begin{alignat}{2}
\ell_c^{OT} = &\!\min_{\pi \in \{0,1/n\}^{n\times K}} &\:& \sum_{i=1}^n  \sum_{k=1}^K  \| f_\theta(x_i) - \mu_k \|^2 \pi_{k,i} \label{eq:OT_reg} \tag{$OT$}\\
&\qquad \text{s.t.} &      & \pi \One_{K} = \frac{1}{n} \One_{n},\label{eq:constraint1} \tag{$c_1$}\\
&                  &      & \pi^\top \One_{n}=w,\label{eq:constraint2} \tag{$c_2$} 
\end{alignat}
\end{subequations}
where  $w = (\frac{n_1}{n},\dots,\frac{n_k}{n}) \in \Delta_K $ is the vector of cluster proportions.

The above problem is known as optimal transport between the discrete measure $\alpha \eqdef \frac{1}{n}\sum_{i=1}^n \delta_{f_\theta(x_i)} $ and $\beta = \sum_{k=1}^K \frac{n_k}{n} \delta_{\mu_k}$. If we remove the constraint on cluster sizes \eqref{eq:constraint1}, it boils down to the objective function of the $k$-means problem with cluster centers $\cbr{\mu_1, \ldots, \mu_K}$ \citep{pollard1982quantization}. 

\begin{algorithm}[t]
\caption{Sinkhorn's Algorithm for Reg. OT}\label{algo:sink}
%\hspace*{\algorithmicindent} \textbf{Input} \\
%\hspace*{\algorithmicindent} \textbf{Output}
\begin{algorithmic}[1]
\State \textbf{Parameters} $\epsilon$ ; $n_{iter}$
\State \textbf{Input} $(f_\theta(x_i))_{i=1\dots n}$ ;  $(\mu_k)_{k=1\dots K}$ ;  $w$ 
 \State $C_{ik} = \norm{f_\theta(x_i) - \mu_k}^2 \quad \forall \: i,k$
 \State $M = \exp(-C/\epsilon)$
 \State Initialize $b \leftarrow \One_K$
 \For{$j  = 1,2,\dots, n_{iter}$}
 	\State $a \leftarrow \frac{1}{n}\frac{\One_n}{M b} $
 	\State $b \leftarrow \frac{w}{M^\top a} $
 \EndFor
 	\State \textbf{Return} $\pi_{ik} = a_i M_{ik} b_k \quad \forall \: i,k $
\end{algorithmic}
\end{algorithm}

\paragraph{Entropic regularization of  optimal transport} Solving optimal transport is computationally expensive as it requires solving a large linear program and a common workaround in the literature is to regularize the problem with entropy \citep{cuturi2013sinkhorn}. The regularized problem then reads as follows:
\begin{subequations}
\begin{alignat}{2}
\ell_c^{OT_\epsilon} = &\!\min_{\pi \in [0,1]^{n\times K}} &\:& \sum_{i=1}^n  \sum_{k=1}^K  \| f_\theta(x_i) - \mu_k \|^2 \pi_{k,i} \nonumber \\
&&& \qquad + \epsilon \pi_{k,i}\left(\log(\pi_{k,i}) -1\right) \label{eq:OT} \tag{$OT_\epsilon$}\\
&\qquad \text{s.t.} &      & \pi \One_{K} = \frac{1}{n}\One_{n},\label{eq:constraint1reg} \tag{$c_1$}\\
&                  &      & \pi^\top \One_{n}=w,\label{eq:constraint2reg} \tag{$c_2$} 
\end{alignat}
\end{subequations}
The addition of entropy allows to solve the problem with a much faster iterative algorithm, called Sinkhorn's algorithm, whose iterations are summarized in Algorithm~\ref{algo:sink}. Although this fast solver is the main reason why regularized  optimal transport became routinely used in machine learning tasks, recent papers have exploited the fact that it also leads to a differentiable loss, whose gradients can be easily computed with backpropagation through Sinkhorn iterations \citep{genevay2017learning,salimans2018improving}.

It is known that a linear program reaches its maximum on the vertices, which is why the Optimal Transport problem is equivalent to its relaxation to the simplex.  The addition of entropy will move the solution away from the optimal vertex, towards the center of the constraint polytope thus yielding smoother assignments \citep{peyre2019computational}. This is formalized in the proposition below:
\begin{prop}
Consider the regularized optimal transport problem \eqref{eq:OT_reg}, and the optimal assignment $\pi_\epsilon$.\\
When $\epsilon \rightarrow 0$ :  
\begin{itemize}
\item $\pi_\epsilon \rightarrow \pi$ (the solution of \eqref{eq:OT_reg}) \item 
$\ell_c^{OT_\epsilon} \rightarrow \ell_c^{OT}$, 
\end{itemize}
When $\epsilon \rightarrow \infty$ : 
\begin{itemize}
\item  $\pi_\epsilon \rightarrow \frac{1}{n}\One_n w$ (i.e., each point is assigned to all clusters according to global proportions w) 
\item $\ell_c^{OT_\epsilon} \rightarrow \frac{1}{n} \sum_{i=1}^n  \sum_{k=1}^K w_k \| f_\theta(x_i) - \mu_k \|^2$.
\end{itemize}
\end{prop}

\begin{proof} The proposition is an adaptation of Theorem 1 in \citep{genevay2017learning} to our clustering setting.
\end{proof}

The choice of $\epsilon$ is a crucial question: when epsilon gets smaller -- i.e. when we get closer to 'true' Optimal Transport -- Sinkhorn's algorithm requires more iterations to converge (see e.g. \citep{peyre2019computational}) meaning that a better approximation of  optimal transport comes at a heavy computational price. However, it has recently been proved that approximating  optimal transport from samples -- which is typically the case in machine learning, it is actually beneficial to use $\epsilon$ not too small to avoid the curse of dimension from which optimal transport suffers \citep{genevay2018sample}.

\begin{algorithm}[t]
\caption{OT-based Deep Clustering}\label{algo:cluster}
%\hspace*{\algorithmicindent} \textbf{Input} \\
%\hspace*{\algorithmicindent} \textbf{Output}
\begin{algorithmic}[1]
\State \textbf{Parameters} $K$, $n_{pre-train}$, $n_{epochs}$, m 
\State \textbf{Input} Dataset $(x_1,\dots, x_n)$, cluster proportions $w$
\State \textbf{Initialize} $f_\theta$ (encoder) and $g_\phi$ (decoder) with random weights
\State \textbf{Initialize} centers $\mu$ with $k$-means on embedded images $(f_\theta(x_1),\dots, f_\theta(x_n))$
\For{$i=1$ to $n_{pre-train}$} \quad (pre-training)
\For{$j=1$ to $n/m$}
\State $\Dcal_j = (x^{(j)}_1, \dots, x^{(j)}_{m})$ batch of size $m$
\State Compute loss $\ell_r^{ae}(\theta,\phi)$
\State Update $\theta$, $\phi$ with a gradient step
\EndFor
\EndFor
\For{$i=1$ to $n_{epochs}$} \quad (Training)
\For{$j=1$ to $n/m$}
\State $\Dcal_j = (x^{(j)}_1, \dots, x^{(j)}_{m})$ batch of size $m$
\State Compute $\pi(f_\theta(\Dcal_j),\mu,w)$ with Sinkhorn
\State Compute loss $\ell_r^{ae}(\theta,\phi)+\ell_c^{OT_\epsilon}(\theta,\mu)$
\State Update $\theta$, $\phi$ and $\mu$ with a gradient step
\EndFor
\EndFor

\For {$i=1$ to $n$} \quad (Final Clustering)
\State Assign $x_i$ to $k_i = \argmin_k \norm{f_\theta(x_i)-\mu_k}^2$ 
\EndFor
%\State \textbf{return} 
\end{algorithmic}
\end{algorithm}

\begin{figure*}[ht]
\centering
\includegraphics[width =.49\textwidth]{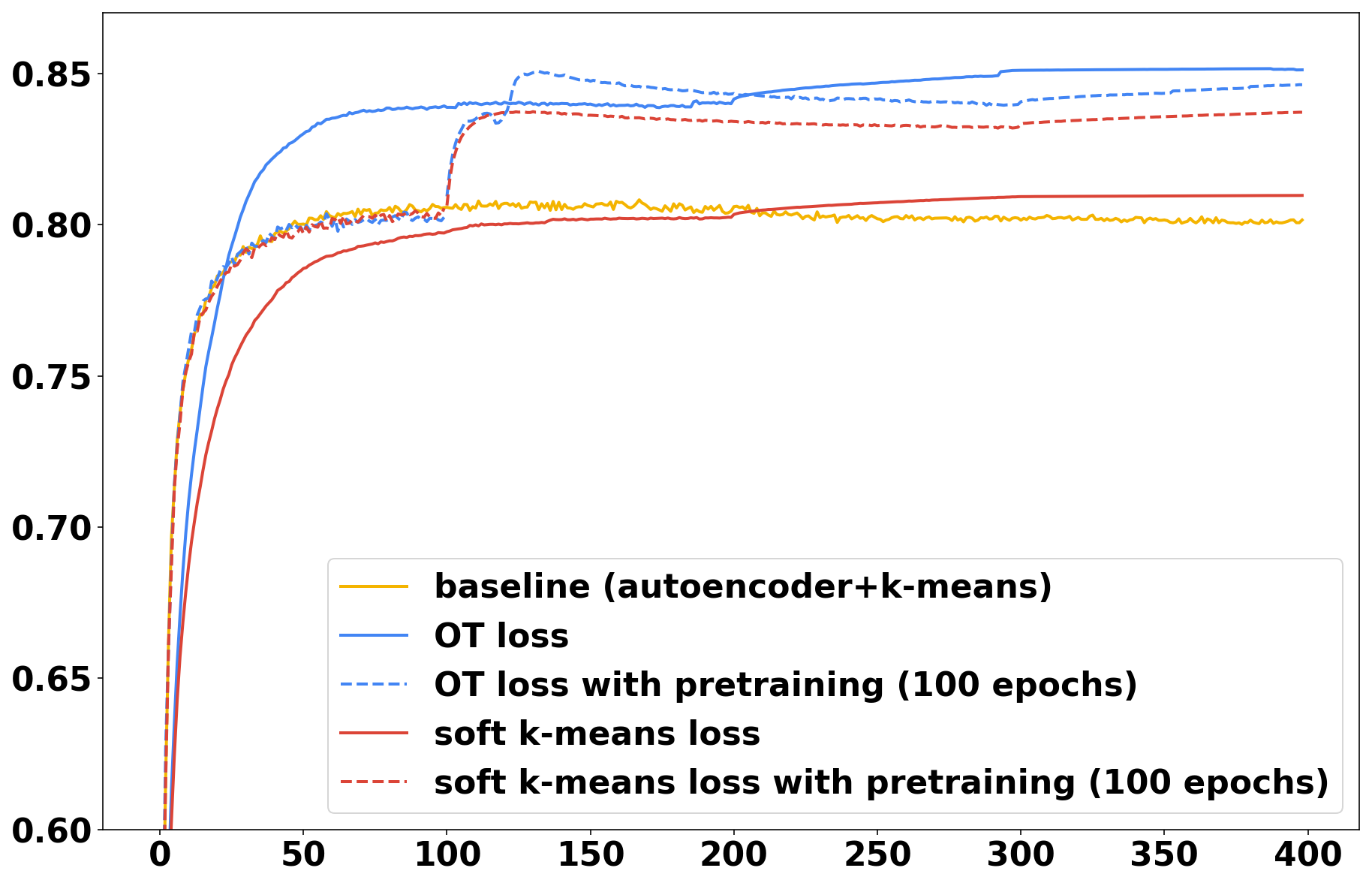}
\includegraphics[width =.49\textwidth]{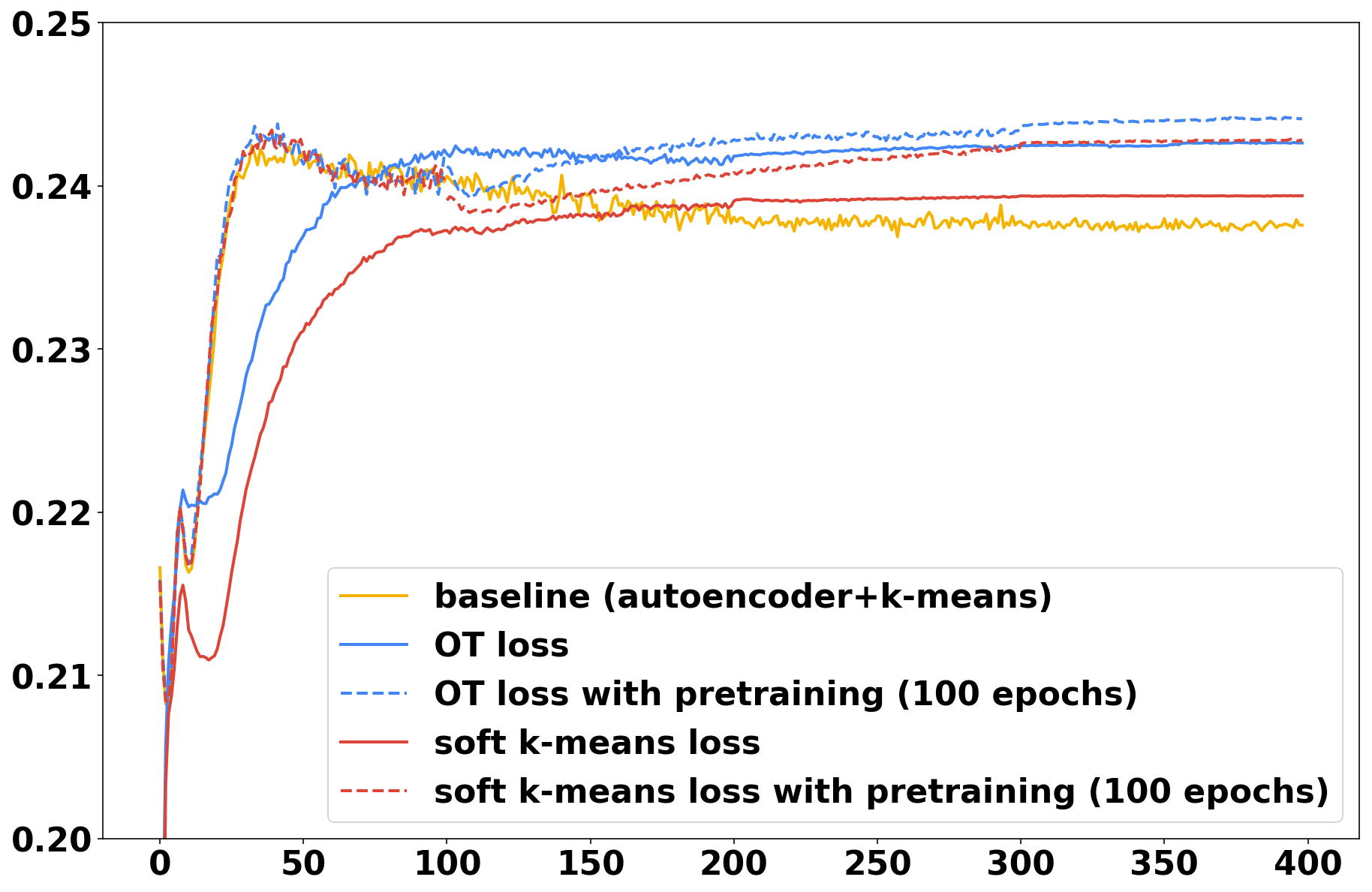}
\caption{Accuracy on MNIST (left) and CIFAR10 (right)}
\label{fig:acc}
\end{figure*}

\paragraph{Link with soft-assignments in $k$-means}
The optimal transport formulation includes two marginal constraints, one being that each sample is assigned to exactly one cluster and the other being the cluster size. The latter constraint can be omitted to obtain an objective which is that of $k$-means. When regularizing the optimal transport problem with only one marginal constraint, we thus get a differentiable $k$-means objective.
\begin{prop}\label{prop:softkmeans}
Consider the variant of entropy-regularized Optimal Transport with only one marginal constraint (i.e., no prior on cluster sizes):
\begin{subequations}
\begin{alignat}{2}
 &\!\min_{\pi \in [0,1]^{n\times K}} &\:& \sum_{i=1}^n  \sum_{k=1}^K  \| f_\theta(x_i) - \mu_k \|^2 \pi_{k,i} \nonumber \\
&&& \qquad + \epsilon \pi_{k,i}\left(\log(\pi_{k,i}) -1\right)  \nonumber \\ %\tag{$OT_\epsilon$}\\
&\qquad \text{s.t.} &      & \pi \One_{K} = \frac{1}{n}\One_{n}, \nonumber %\tag{$c_1$}
\end{alignat}
\end{subequations}
 then the optimal assignment $\pi^*$ is given by 
 \begin{equation}\pi_{k,i}^* = \frac{e^{- \norm{f_\theta(x_i)-\mu_k}^2}/\epsilon}{n \sum_{k'=1}^K e^{- \norm{f_\theta(x_i)-\mu_{k'}}^2}/\epsilon}. \label{eq:optimal_assign}
 \end{equation}
\end{prop}

\begin{proof} This is a convex function of $\pi$ with linear constraints.
Denoting by $\lambda$ the Lagrange multiplier for the constraint, the Lagrangian is :
\begin{align}
\Lcal(\pi, \lambda) = & \sum_{i=1}^n  \sum_{k=1}^K  \| f_\theta(x_i) - \mu_k \|^2 \pi_{k,i}  \\ &+ \epsilon \pi_{k,i} \left(\log(\pi_{k,i}) -1\right) \\ &+ \sum_{i=1}^n \lambda_i (\sum_{k=1_K}\pi_{ik} - \frac{1}{n})
\end{align}
The first order conditions of the Lagrangian gives \eqref{eq:optimal_assign}.
\end{proof}

 The solution of that problem corresponds exactly to the differentiable $k$-means loss introduced by \citet{fard2018deep}, which the authors motivated by replacing the $\min$ in the $k$-means objective \eqref{eq:losskm} with the softmin function. Hence Proposition~\ref{prop:softkmeans} provides a new interpretation of DKM, and shows that the approach we propose below generalizes DKM by adding constraints on the cluster sizes.

\begin{figure*}
\centering
\includegraphics[width =.49\textwidth]{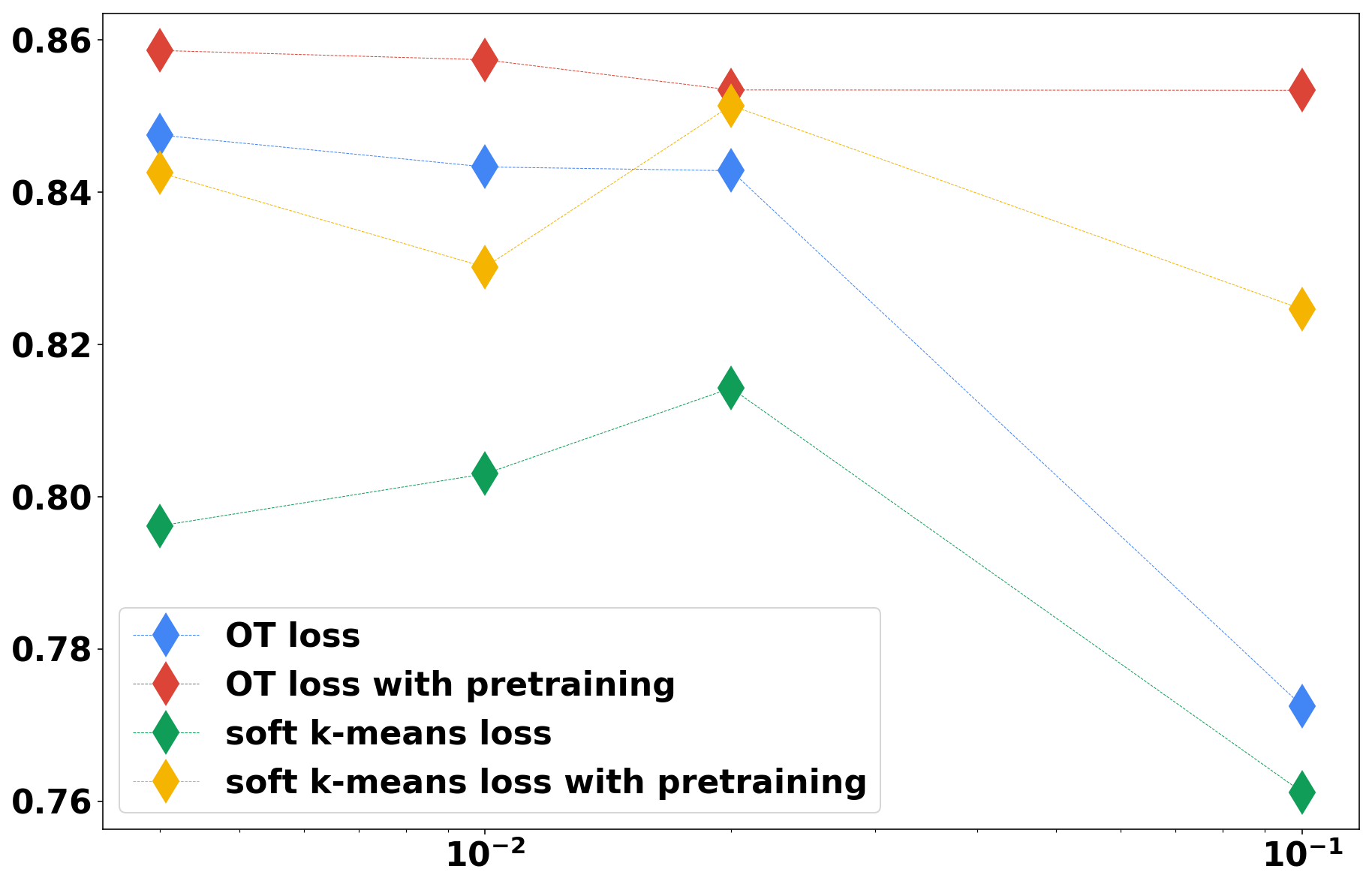}
\includegraphics[width =.49\textwidth]{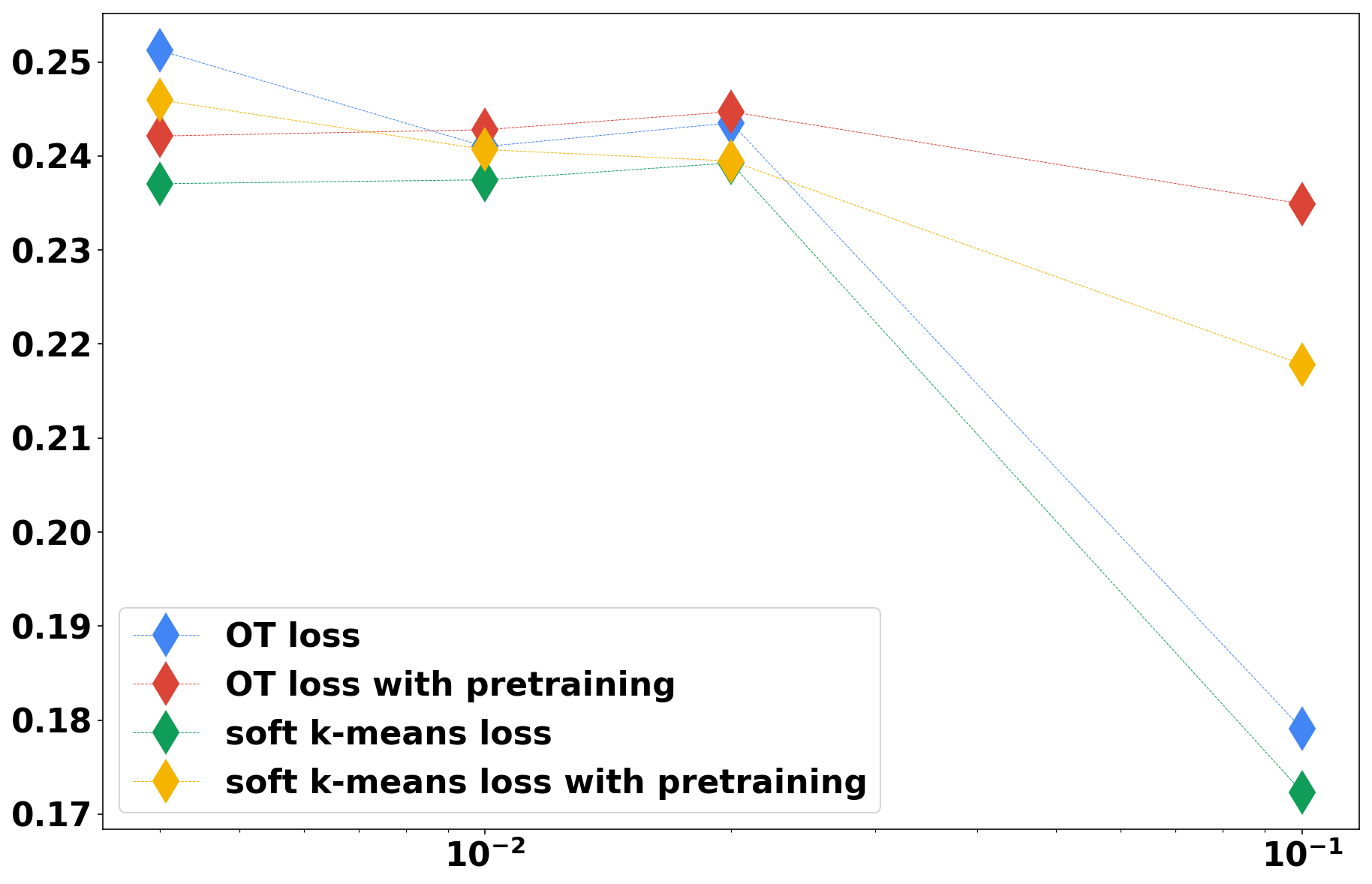}
\caption{Accuracy after 200 epochs (averaged over 5 runs) as a function of $\epsilon$ for MNIST (left) and CIFAR10 (right)}
\label{fig:eps}
\end{figure*}

\paragraph{Solving the clustering problem}
Several papers in the literature use the $k$-means objective as the clustering loss $\ell_c$ in \eqref{eq:combloss}. We propose to replace this by the objective function of regularized optimal transport, which allows to:
\begin{enumerate}[label=(\roman*)]
\item enforce a prior on the cluster proportions,
\item obtain a loss that is differentiable with respect to both the cluster centers $(\mu_k)_k$ and the embedding parameters $\theta$.
\end{enumerate}
The clustering problem \eqref{eq:combloss} becomes:
\begin{equation}
\min_{\theta,\phi,\mu} \ell_{r}^{ae}(\theta,\phi) + \ell_c^{OT_\epsilon}(\theta,\mu)\,,
\end{equation}
where $\ell_{r}^{ae}$ is the reconstruction loss of the auto-encoder defined in \eqref{eq:lossAE} and $\ell_c^{OT_\epsilon}$ is the result of the regularized optimal transport problem \eqref{eq:OT_reg}.  This function is differentiable with respect to all its variables, and can be minimized with SGD. We outline the full learning procedure in Algorithm~\ref{algo:cluster}.

\section{Experiments}

We assess the efficiency of our optimal transport-based clustering loss on two classical benchmark datasets for image classification, MNIST and CIFAR10. We compare it against a simple baseline which consists in learning an embedding with an auto-encoder and then running $k$-means on the embedded data, and against the state-of-the-art DKM method  of \citet{fard2018deep} based on the soft $k$-means loss, which already proved its superiority over existing methods.

\paragraph{Experimental setting}
For the MNIST dataset, we follow mutiple examples from the litterature by using a fully connected auto-encoder with 500-250-10-250-500 structure and ReLu activation. For the CIFAR dataset, we use as an encoder a standard convolutional network using ReLu activation, with three convolution layers of respective depths 32, 64, 128, respective kernel sizes 5, 5, 3, and common stride of 2, followed by a fully connected layer to a latent space of dimension 10. In both cases we use batches of size $m=300$. The gradient updates are made with the Adam algorithm and the standard learning rate from TensorFlow (0.001) with step decay, as there should not be parameter tuning in an unsupervised setting.
The algorithm used for training is summarized in Algorithm \ref{algo:cluster}. We also run $k$-means on raw pixels, to give an idea of how much structure is induced in the data by the embedding $f_\theta$. Following usual guidelines regarding regularization for optimal transport, we set $\epsilon = 10^{-2}$ which gives a good enough approximation of optimal transport without requiring too many Sinkhorn iterations. We show in Fig.~\ref{fig:eps} that the final performance is robust to the choice of $\epsilon$ as long as it is not too large.

\begin{table*}[ht]
\begin{center}
\begin{tabular}{|c|c|c|c|c|c|c|}
 \hline &&&&&&\\
 & $k$-means & AE + $k$-means & soft $k$-means & soft $k$-means (p) & OT & OT (p)   \\ &&&&&&\\ \hline &&&&&&\\
 MNIST & 0.513 & $0.801(\pm 0.025)$ & $0.810(\pm 0.033)$ & $0.837(\pm 0.032)$ & $0.851(\pm 0.032)$ & $0.846(\pm 0.040)$ \\
  & & $[0.765, 0.912]$ & $[0.677, 0.883]$ &   $[0.781, 0.923]$ & $[0.771, 0.932]$ & $[0.759, 0.928]$

 \\ &&&&&&\\ \hline &&&&&&\\
CIFAR10 & 0.206 & $0.237(\pm 0.005)$ &  $0.239(\pm 0.008)$ &   $0.243(\pm 0.010)$ & $0.243(\pm 0.007)$ & $0.244(\pm 0.009)$  \\
 & & $[0.230, 0.259]$ & $[0.227, 0.257]$  & $[0.227, 0.261]$ &  $[0.232, 0.260]$ & $[0.230, 0.266]$

  \\ &&&&&&\\ \hline

\end{tabular}
\end{center}
\caption{Average accuracy from clustering on CIFAR and MNIST datasets (over 50 runs) with standard deviation and max and min accuracy over the runs (second line). (p) means `with pre-training'} \label{tab:results}
\end{table*}

\OMIT{
\begin{table}
\begin{center}
\begin{tabular}{|c|c|c|c|c|}
 \hline &&\\
 & MNIST & CIFAR10  \\ &&\\ \hline &&\\
 $k$-means (raw ) & 0.513 & 0.206 \\ &&\\\hline &&\\
 AE + $k$-means & $0.801(\pm 0.025)$
 & $0.237(\pm 0.005)$ \\
 & $[0.765, 0.912]$
 & $[0.230, 0.259]$
 \\ &&\\ \hline &&\\
 soft $k$-means & $0.810(\pm 0.033)$
 & $0.239(\pm 0.008)$ \\
 & $[0.677, 0.883]$
& $[0.227, 0.257]$

  \\ &&\\\hline&&\\
 soft $k$-means  & $0.837(\pm 0.032)$
 & $0.243(\pm 0.010)$ \\ (pre-trained)
 & $[0.781, 0.923]$
 & $[0.227, 0.261]$

  \\&&\\ \hline&&\\
 OT & $0.851(\pm 0.032)$&$0.243(\pm 0.007)$
 \\
 & $[0.771, 0.932]$
& $[0.232, 0.260]$\\ &&\\ \hline &&\\
 OT & $0.846(\pm 0.040)$
 & $0.244(\pm 0.009)$ \\ (pre-trained)
 &$[0.759, 0.928]$
 &$[0.230, 0.266]$
  \\ &&\\ \hline

\end{tabular}
\end{center}
\caption{Average accuracy from clustering on CIFAR and MNIST datasets (over 50 runs) with standard deviation and max and min accuracy over the runs (second line).} \label{tab:results}
\end{table}}

\paragraph{Evaluation}
After each epoch, we evaluate the different methods by computing the accuracy given by matching clusters to classes through the following formula:
$$ accuracy = \max_{\sigma \in \mathfrak{S}}\sum_{i=1}^n \One_{y_i = \sigma(k_i)}, $$
where $y_i$ and $k_i$ are respectively the class label and the cluster index associated to $x_i$ and  $\mathfrak{S}$ is the set of permutations of $\{1\dots K\}$. For the `AE + $k$-means' method, the cluster assignment is made by running $k$-means on the embedded data, while for both `soft $k$-means' and `OT', since these methods also learn the cluster centers $(\mu_1,\dots,\mu_K)$ we assign the point $x_i$ to cluster $k_i$ such that $k_i = \argmin_k \norm{f_\theta(x_i)-\mu_k}^2$.
The optimal matching between clusters and classes is done via the Hungarian algorithm, as in the literature.

The evolution of accuracy during training for all three methods (auto-encoder + kmeans, soft $k$-means, optimal transport) is plotted in Fig.~\ref{fig:acc}. The curves are averaged over 50 runs. The final accuracies are reported in Table~\ref{tab:results} along with the standard deviations. We can seen that the optimal transport loss achieves superior accuracy, but mostly doesn't need to rely on pre-training to get good performance, contrarily to soft $k$-means, whose performance is only slightly above the baseline without pre-training. To assert the statistical significance of the superiority of optimal transport in this framework, we further run a Welch's t-test over the final accuracies in the 50 runs. Without pre-training, we find out that optimal transport is significantly better than soft $k$-means (p-value of $0.0067$ for CIFAR10 and $10^{-12}$ for MNIST). With pre-training, optimal transport is still significantly better than soft $k$-means for MNIST at the 10\% level (p-value of 0.10) but it's not the case for CIFAR10 (p-value of 0.33). Note that for both datasets, pre-training did not yield significantly better performance for optimal transport (p-values $> 0.2$), while it significantly improves soft $k$-means (p-values $< 0.05$).  

Fig.~\ref{fig:eps} displays the accuracy after 200 epochs for each method, as a function of $\epsilon$. We see that the competitive advantage of OT over soft $k$-means is robust to the choice of $\epsilon$ as long as it is not too large. Incidentally, the methods with pre-training are more robust to large values for $\epsilon$. Note that these curves are averaged over only 5 runs and thus can not be regarded as statistically significant, they merely serve as a proof of robustness of the method to the chosen parameter.

\section{Conclusion and discussion}

In this paper we propose a new fully differentiable framework for deep clustering, based on regularized optimal transport, which generalizes the recently proposed approach of \citet{fard2018deep} based on soft-$k$-means. Its main advantage over competing methods is its ability to naturally enforce a prior on class proportions. This significantly improves performance on datasets with well balanced classes, without relying on pre-training of the embedding. 

In our experiments we observed a benefit over soft-$k$-means in situations where the classes are balanced. An interesting direction to explore is to extend the application of our method when the prior knowledge on cluster size is not uniform. This may be relevant in cases, for example, when an expert provides a rough estimate of the proportion of different classes, such as the proportion of cancer cells in an histopathological image. While our formulation lends itself naturally to non-uniform cluster proportions, we observed in preliminary experiments that it performs poorly if no care is taken to make sure that the cluster size constraints ($w$ in Algorithm~\ref{algo:cluster}) is coherent with the set of cluster centers (vectors $\mu_k$ in Algorithm~\ref{algo:cluster}). We found in particular that this is often poorly achieved by the initialization of the centers via $k$-means (step 4 in Algorithm~\ref{algo:cluster}. For instance, consider the case where we have two clusters -- say images of ones and images of twos -- and we know that we should have 20 \% of the former and 80 \% of the latter. However, if the $k$-means initialization of the centers puts the first center in the middle of twos and the second center in the middle of ones, the algorithm will try to enforce a 80\% proportion on ones and 20\% proportion on twos. Generally speaking, to ensure that we are enforcing the cluster proportions properly, some sort of matching has to be done  before the learning phase between the indexes of the clusters and the indexes of the classes. This could be done in a supervised way, by using an example from each class to initialize the centers. This extension falls in the framework of one-shot learning, with an additional knowledge on class proportions.

Another extension of our method would be relax the strict constraint of cluster proportions to a soft constraint, using for example unbalanced optimal transport with a relaxed version of the Sinkhorn algorithm which penalizes the marginal constraints instead of enforcing them strongly \citep{chizat2018scaling}. This may be particularly relevant when small batches are considered, as one would not expect the composition of each batch to perfectly reflect the overall composition.

Finally, we note that our formulation of the clustering problem with optimal transport is closely linked to that of \citet{Dulac-Arnold2019Deep}, who propose an algorithm to learn a classifier from label proportions in mini-batches. The main difference is that instead of using $f_\theta$ to parametrize an embedding, the authors directly use it to predict the probability of belonging to class $k$. The last layer of $f_\theta$ is a softmax, and thus the term $\| f_\theta(x_i) - \mu_k \|^2$ is replaced by $f_{\theta}(x_i)_k$. Besides, they loosen the marginal constraint prescribing the clusters proportions by using unbalanced optimal transport. The latter can also be implemented in our proposed method, as it consists in using a variation of Sinkhorn's algorithm \citep{chizat2018scaling}. However, the performance that they report for large batch sizes in lower than what we report for the fully unsupervised task in our experiments. This would make our optimal transport approach a good candidate for the learning with proportions problem.

%\newpage
\bibliography{bibli}
\bibliographystyle{plainnat}

\end{document}